\documentclass[letterpaper, 10 pt, conference]{ieeeconf}  

\IEEEoverridecommandlockouts                              

\usepackage{graphics} 
\usepackage{epsfig} 
\usepackage{epstopdf}
\usepackage{bm}
\usepackage{amsfonts}
\usepackage{amsmath}
\usepackage{subfigure}
\usepackage[ruled]{algorithm2e}
\usepackage{booktabs} 
\usepackage{cite}
\usepackage{geometry}
\newtheorem{theorem}{Theorem}

\newtheorem{assumption}{Assumption}
\title{\LARGE \bf
Robust Dynamic Control Barrier Function Based Trajectory Planning for Mobile Manipulator
}

\author{Lihao Xu, Xiaogang Xiong, Bai Yang, Yunjiang Lou
\thanks{Lihao Xu, Xiaogang Xiong and Yunjiang Lou are with the School of Mechanical Engineering and Automation,
Harbin Institute of Technology Shenzhen, Shenzhen 518055, P.~R.~China.
{\tt\small \{22s053080@stu., xiongxg@, louyj@\}hit.edu.cn}}
\thanks{Yang Bai is with the Graduated School of Information Science and Technology,
Osaka University 1-5 Yamada, Suita, Osaka 565-0871, Japan
(\tt\small ybai@ist.osaka-u.ac.jp)}
}

\begin{document}

\maketitle
\thispagestyle{empty}
\pagestyle{empty}

\begin{abstract}

High-dimensional robot dynamic trajectory planning poses many challenges for traditional 
planning algorithms. Existing planning methods suffer from issues such as long computation times,
limited capacity to address intricate obstacle models, and lack of consideration for 
external disturbances and measurement inaccuracies in these high-dimensional systems. 
To tackle these challenges, this paper proposes a novel trajectory planning approach that 
combines Dynamic Control Barrier Function (DCBF) with a disturbance observer to create a 
Robust Dynamic Control Barrier Function (RDCBF) planner. This approach successfully plans 
trajectories in environments with complex dynamic obstacles while accounting for external 
disturbances and measurement uncertainties, ensuring system safety and enabling precise
obstacle avoidance. Experimental results on a mobile manipulator demonstrate outstanding
performance of the proposed approach.

\end{abstract}

\section{Introduction}

As robots become more integrated into daily life, trajectory planning faces significant challenges in ensuring 
safety during human-robot interactions in complex and dynamic scenarios \cite{HRI1}. Commonly used algorithms, 
such as precomputed signed distance fields \cite{GPMP, CHOMP, TrajOpt} and sampling-based methods \cite{RRT1,RRT2,RRT3}, 
are suitable for static or minimally changing environments. 
These algorithms generate an initial path that serves as a reference for local planners within the planning framework. 
Local planners like Artificial Potential Fields (APF) \cite{APF} and Dynamic Window Approach (DWA) \cite{DWA} 
handle dynamic obstacles but face theoretical limitations when naturally extending to high-dimensional state spaces. 
For example, the APF algorithm needs to map the potential map into the manipulator's configuration space while the 
size of velocity space in DWA grows exponentially with the number of states. Model predictive control (MPC) 
\cite{MPC1,MPC2,MPC3,MPC4} has been successfully applied for real-time planning in high-dimensional system states; 
however, computational time becomes impractical as symbolic parameters increase such as obstacle numbers or prediction steps.

Control Barrier Functions (CBFs) have become increasingly popular for ensuring safety and have been used in 
trajectory planning problems \cite{CBF1,CBF2,CBF3}. Dynamic Control Barrier Functions (DCBFs) \cite{DCBF} were 
introduced to address obstacle avoidance for rapidly moving obstacles, but they are mainly applied in low-dimensional 
spaces. Some work has successfully applied CBF algorithms to high-dimensional manipulators \cite{HRI2}, 
focusing on collision avoidance within the context of human-robot interaction. However, prior CBF algorithms either 
simplify obstacles as spheres, struggle with dynamic environments or are limited to low-dimensional spaces. 
Additionally, these algorithms rely on precise system models that neglect uncertainties from external disturbances 
and measurement errors in the environment, which prevents practical applications.

\begin{figure}
    \centering
    \includegraphics[width=0.45\textwidth]{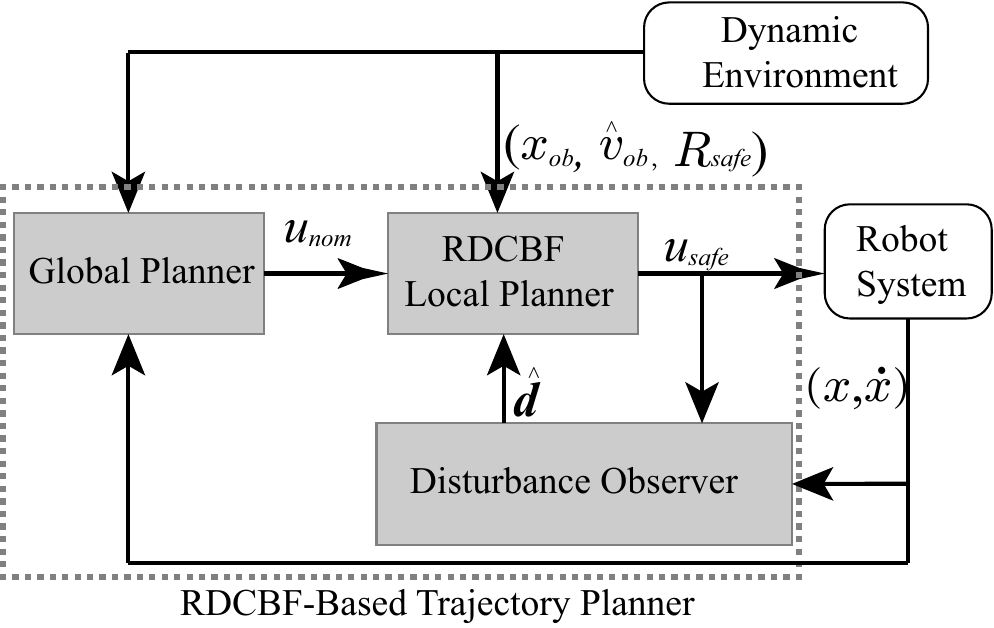}
    \caption{The proposed RDCBF-Based trajectory planner. 
  It utilizes disturbance estimates and potentially inaccurate environment information.
By employing RDCBF, a set of safety constraints are constructed to compensate for both uncertainties.
These constraints are integrated into the input $\bm{u}_{nom}$ from the global planner.
This integration allows for the generation of control commands $\bm{u}_{safe}$ that guarantee safe trajectory planning.}
    \label{structure}
    \end{figure}

While there have been studies on robust version of CBF \cite{RCBF,DOBCBF} to tackle uncertainties in system states, 
no CBF-based algorithm can currently ensure safety when measurement errors are present in the environment. To address 
limitations in handling high-dimensional planning, dynamic environments, complex obstacles, and system disturbances 
and measurement errors, this paper proposes a trajectory planner that relies on Robust Dynamic Control Barrier Function 
(RDCBF), as shown in Fig.~\ref{structure}.

The primary contributions of this study are:
\begin{itemize}
\item A safety function is designed for a mobile manipulator, which reduces wasted free space compared with algorithms that simply employ bounding spheres.
\item DCBF is leveraged to overcome the limitation of traditional CBF algorithms in handling dynamic obstacles during trajectory planning. Based on the DCBF, the RDCBF algorithm is developed to address the dynamic trajectory planning problem in the presence of external disturbances and measurement uncertainty.
\item The proposed RDCBF algorithm is applied to a high-dimensional manipulator, demonstrating their utility and real-time performance.
\end{itemize}

\section{Preliminaries}
This section reviews the main theories for RDCBF
trajectory planner design.
First, the section covers the fundamentals of CBF and DCBF; for more details,
refer to  \cite{DCBF} and \cite{CBFdetail}.
Then, the section provides a brief introduction to the Disturbance Observer (DOB) 
following the results presented in \cite{DOBCBF} and \cite{DOB}.

\subsection{Dynamic Control Barrier Function}
Consider the following nonlinear control affine system
\begin{equation}
\label{syseq}
\bm{\dot{x}}=\bm{f}(\bm{x})+\bm{g}(\bm{x})\bm{u},
\end{equation}
where $\bm{x} \in \mathcal{D} \subseteq \mathbb{R}^n$ represents the state belonging to the state space $\mathcal{D}$, 
and $\bm{u} \in \mathcal{U} \subseteq \mathbb{R}^m$ denotes the control input within the admissible control input set $\mathcal{U}$.
The system functions $\bm{f}(\cdot)$
and $\bm{g}(\cdot)$ are assumed to be locally Lipschitz continuous.
Consider the system~(\ref{syseq}) and given a compact set $\mathcal{C}$ defined by
\begin{equation}
\label{safeset}
\mathcal{C}:=\{\bm{x} \in \mathbb{R}^n : h(\bm{x}) \geq 0\},
\end{equation}
where $h(\bm{x})$ is a continuously differentiable function.
The set $\mathcal{C}$ is reffed to $\textit{safe set}$ in this paper.
The function $h(\cdot)$, parameterized by the state $\bm{x}$, reflects the safety of system~(\ref{syseq}).
If there exists a constant $\gamma$ such that: 
\begin{equation}
\label{cbfdef}
\sup_{\bm{u} \in \mathcal{U}} \left[ L_fh(\bm{x})+ L_gh(\bm{x})\bm{u} + \gamma h(\bm{x})
\right] \geq 0,\forall x \in \mathcal{D},
\end{equation}
where $L_fh(\bm{x})=\frac{\partial h}{\partial \bm{x}}\bm{f}(\bm{x})$, 
$L_gh(\bm{x})=\frac{\partial h}{\partial \bm{x}}\bm{g}(\bm{x})$
are the Lie derivatives of $h(\cdot)$ along $\bm{f}$ and $\bm{g}$ respectively.
Then $h(\bm{x})$ is called a (Zeroing) CBF of system~(\ref{syseq}).
Once given a CBF $h(\bm{x})$, the set of all control inputs that satisfy~(\ref{cbfdef})
for every $\bm{x} \in \mathcal{D}$ can be defined as:
\begin{equation}
\label{CBFconstrain}
K_{cbf}(\bm{x})=\{ \bm{u} \in \mathcal{U}: L_fh(\bm{x})+ L_gh(\bm{x})\bm{u} + \gamma h(\bm{x}) \geq 0 \}.
\end{equation}
It has been proven in \cite{RCBF} that any Lipschitz continuous controller 
$\bm{u}(\bm{x}) \in K_{cbf}(\bm{x})$ will
guarantee the forward invariance of $\textit{safe set}$ $\mathcal{C}$. 

\indent 
In a dynamic environment, the CBF $h(\bm{x})$ needs to incorporate information about the environment. 
Therefore, the Dynamic Control Barrier Function (DCBF), denoted as $h(\bm{x},\bm{x}_{ob})$
and simply represented as $h$, involves environment information.
Here, $\bm{x}_{ob}$ represents the state of the obstacle. 
The condition in~(\ref{CBFconstrain}) is reconstructed as follows:

\begin{equation}
\label{DCBFconstrain}
\sup_{\bm{u} \in \mathcal{U}} \left[ L_fh + L_gh\bm{u} +L_ph\bm{v}_{ob}
+ \gamma h \right] \geq 0,
\end{equation}
where $L_ph\bm{v}_{ob}:=\frac{\partial h}{\partial \bm{x}_{ob}}\frac{\partial \bm{x}_{ob}}{\partial t}$
and $\bm{v}_{ob}=\frac{\partial \bm{x}_{ob}}{\partial t}$ ,
coresponding safe control set of DCBF is defined as:
\begin{equation}
\label{KDCBF}
K_{dcbf}(\bm{x},\bm{x}_{ob})=\{ \bm{u} \in \mathcal{U}: 
L_fh+ L_gh\bm{u} + L_ph\bm{v}_{ob} +\gamma h \geq 0 \}.
\end{equation}
Any Lipschitz continuous controller $\bm{u}(\bm{x},\bm{x}_{ob}) \in K_{dcbf}$ will 
guarantee the forward invariance of the $\textit{safe set}$.

\subsection{Disturbance Observer (DOB)}
Consider system~(\ref{syseq}) that is disturbed by unknown 
external disturbance $\bm{d} \in \mathbb{R}^q$:
\begin{equation}
\label{sysdeq}
\bm{\dot{x}}=\bm{f}(\bm{x})+\bm{g}(\bm{x})\bm{u}+\bm{g}_d(\bm{x})\bm{d},
\end{equation}
where $\bm{g}_d(\cdot)$ is assumed to be locally Lipschitz continuous.
\begin{assumption}
The disturbance $\bm{d}(t)$ and its derivative $\dot{\bm{d}}(t)$ are bounded by
known positive constants, i.e., $\|\bm{d}(t)\|\leq\omega_0$ and 
$\|\dot{\bm{d}}(t)\|\leq\omega_1,\forall t > 0$ where $\omega_0\geq0$ and $\omega_1\geq0$.
\end{assumption}

Given system~(\ref{sysdeq}), we define the following DOB:
\begin{equation}
\label{DOB}
\begin{cases}
\begin{aligned}
\hat{\bm{d}} &=\bm{z}+\alpha \bm{W}, \\
\dot{\bm{z}} &=-\alpha \bm{\lambda}(\bm{f}+\bm{g}\bm{u}+\bm{g}_d\hat{\bm{d}}),
\end{aligned}
\end{cases}
\end{equation}
where $\hat{\bm{d}}$ is the disturbance estimation, $\alpha > 0$ is a positive tuning parameter,
$\bm{\lambda}(\bm{x})$ is the observer gain matrix
satisfying $-\bm{x}^T\bm{\lambda}\bm{g}_d\bm{x}\leq-\bm{x}^T\bm{x}$ for any $\bm{x} \in \mathcal{D}$
and $\bm{W}(\bm{x})$ is a function satisfying $\frac{\partial\bm{W}}{\partial\bm{x}}=\bm{\lambda}(\bm{x})$.

The disturbance estimation error is defined as:
\begin{equation}
\label{derror}
\bm{e}_d=\hat{\bm{d}}-\bm{d}.
\end{equation}
Then, $\dot{\bm{e}}_d=\dot{\hat{\bm{d}}}-\dot{\bm{d}}=\dot{\bm{z}}+\alpha \bm{\lambda}\dot{\bm{x}}-\dot{\bm{d}}$.
Substituting disturbed system~(\ref{sysdeq}) and DOB~(\ref{DOB}) into $\dot{\bm{e}_d}$ yields:
\begin{equation}
\label{derrorderivate}
\dot{\bm{e}_d}=-\alpha \bm{\lambda}\bm{g}_d\bm{e}_d-\dot{\bm{d}}.
\end{equation}
Choose a Lyapunov candidate function $V=\frac{1}{2}\|\bm{e}_d\|^2$.
The following results have been derived in \cite{DOBCBF}:
\begin{equation}
\label{Vderivate}
\dot{V}\leq-2\kappa V+\frac{\omega_1^2}{2\mu },
\end{equation}
where $\kappa := \alpha - \frac{\mu}{2}$ and $\mu$ is a constant satisfying $0<\mu<2\alpha$.
By comparison lemma~\cite{COMPARE}, $\bm{e}_d(t)$ satisfies:
\begin{equation}
\label{etconverge}
\|\bm{e}_d(t)\| \leq \sqrt{\|\bm{e}_d(0)\|^2\exp(-2\kappa t)+
\frac{\omega_1^2(1 - \exp(-2\kappa t))}
{2\mu\kappa}},
\end{equation}
which gives the conclusion that the disturbance estimation error $\bm{e}_d$ is
uniformly bounded.

\section{RDCBF Design for Trajectory Planning}
\label{sec:RDCBF}
We aim to handle complex dynamic obstacles, rather than simplifying them to spheres, 
and ensure system safety in the presence of disturbance and velocity estimation errors.
Therefore, our first step is to design a proper and fundamental CBF (function $h(\cdot)$)
tailored to intricate obstacles. This CBF is task-dependent and aligned with specific safety definitions 
(such as distance to obstacles in this paper, or joint angles in the inverted pendulum problem). 
Subsequently, employing the theory of RDCBF introduced in this paper,
we address challenges posed by dynamic environments, external disturbance,
and environmental estimation errors.
\subsection{Geometric Computing in Mobile Manipulator}
\label{sec:disturance}
Existing algorithms, such as GJK \cite{GJK}, 
use iterative methods to compute the distance between convex geometric entities. 
However, CBF-based algorithms require analytical expressions for distance.
Many existing planning methods treat the obstacle as a sphere 
and model the manipulator as several capsules or spheres, 
corresponding to the distance between points and line segments. 
In \cite{Geometric}, various methods for distance computation are presented. 
Assuming that the algorithms for calculating distances of point-to-point and point-to-segment
in 3D space are well-established.
For more intricate scenarios,
additional algorithms for geometric computations will be presented below.

\subsubsection {Segment-to-segment Distance in 3D Space}
Suppose there are two line segments $\bm{l}_A, \bm{l}_B \in \mathbb{R}^6$ 
defined by endpoints $(\bm{A}_0, \bm{A}_1)$ and $(\bm{B}_0, \bm{B}_1)$, respectively.
A geometric method is employed to compute the minimum distance between $\bm{l}_A$ and $\bm{l}_B$.
\begin{figure}
  \centering
  \includegraphics[width=0.5\linewidth]{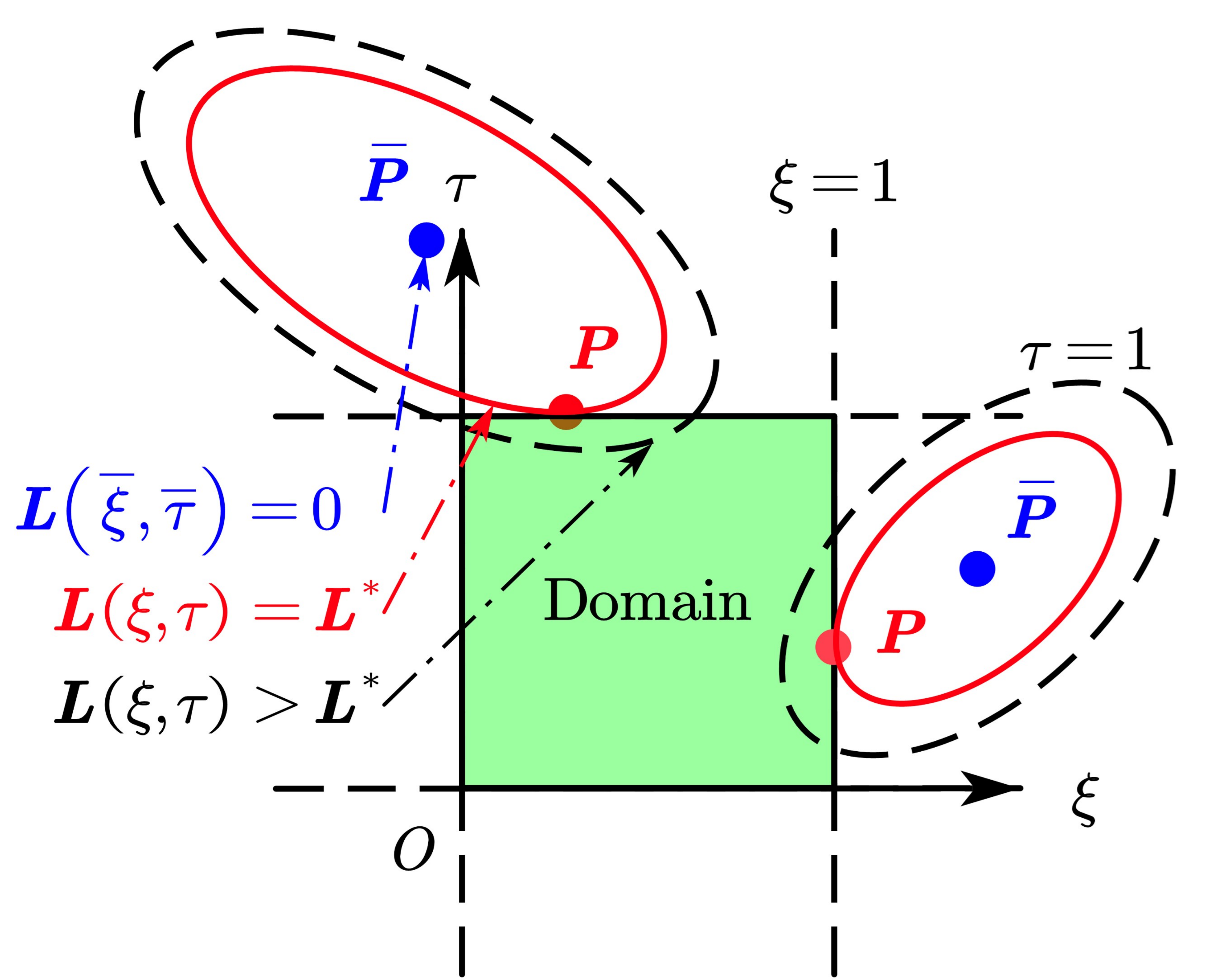}
  \caption{Elliptic level curves $\bm{L}$ and coresponding critical points $\bm{P}$.}
  \label{ltol}
\end{figure}

\indent Let points $\bm{A}(\xi) \in \bm{l}_A$ and $\bm{B}(\tau) \in \bm{l}_B$ :
\begin{equation}
\label{sands}
\begin{aligned}
\bm{A}(\xi) &:=\bm{A}_0+\xi(\bm{A}_1-\bm{A}_0)&,\xi\in [0,1]&
\\
\bm{B}(\tau) &:=\bm{B}_0+\tau(\bm{B}_1-\bm{B}_0)&,\tau\in [0,1]&.
\end{aligned}
\end{equation}
The square distance between $\bm{A}(\xi)$ and $\bm{B}(\tau)$ is given by:
\begin{equation}
\label{Ds2s}
\begin{aligned}
\bm{D}(\xi,\tau)& = \|\bm{A}(\xi)-\bm{B}(\tau)\|_2^2 \\
&=z_1\xi^2-2z_2 \xi \tau+z_3\tau^2+2z_4\xi-2z_5\tau+z_6,\\
\end{aligned}
\end{equation}
where:
\begin{equation}
\label{Ds2swhere}
\begin{aligned}
  z_1&=(\bm{A}_1-\bm{A}_0)\cdot(\bm{A}_1-\bm{A}_0),\\
  z_2&=(\bm{A}_1-\bm{A}_0)\cdot(\bm{B}_1-\bm{B}_0),\\
  z_3&=(\bm{B}_1-\bm{B}_0)\cdot(\bm{B}_1-\bm{B}_0),\\
  z_4&=(\bm{A}_1-\bm{A}_0)\cdot(\bm{A}_0-\bm{B}_0),\\
  z_5&=(\bm{B}_1-\bm{B}_0)\cdot(\bm{A}_0-\bm{B}_0),\\
  z_6&=(\bm{A}_0-\bm{B}_0)\cdot(\bm{A}_0-\bm{B}_0).\\
\end{aligned}
\end{equation}

\indent Let $\Delta :=z_1z_3-z_2^2$, then if $\Delta = 0$,
$\bm{l}_A$ and $\bm{l}_B$ are parallel. The minimum value of $\bm{D}$, denoted as $\bm{D}^*$,
is determined as the smaller of the values derived from $\bm{A}_0$ to $\bm{l}_B$ and $\bm{A}_1$ to $\bm{l}_B$.

If $\Delta \neq 0$, 
we can reconstruct (\ref{Ds2s}) as follows:
\begin{equation}
  \label{Ds2selliptic}
  \begin{aligned}
    \bm{D}(\xi,\tau) &= \bm{L}(\xi, \tau)+\bm{D}_0
  \end{aligned}
  \end{equation}
where:
\begin{equation}
  \label{optimalsolution}
  \begin{aligned}
  \bm{L}(\xi, \tau) &:= z_1(\xi-\bar{\xi})^2-2z_2(\xi-\bar{\xi})(\tau-\bar{\tau})+z_3(\tau-\bar{\tau})^2,\\
  \bm{D}_0&:=z_6-(z_1\bar{\xi}^2+2z_2\bar{\xi}\bar{\tau}+z_3\bar{\tau}^2),\\
  \bar{\xi}&=\frac{z_2z_5-z_3z_4}{\Delta},
  \bar{\tau}=\frac{z_1z_5-z_2z_4}{\Delta}.\\
  \end{aligned}
  \end{equation}
From \eqref{Ds2selliptic}, the value of $\bm{D}$ is decomposed into an elliptic function $\bm{L}(\xi, \tau)$,
the center $\bar{\bm{P}}$ of which is ($\bar{\xi}, \bar{\tau}$), and a constant term $\bm{D}_0$. 
Thus the minimization of $\bm{D}$ is equivalent to the minimization of $\bm{L}$. 
The minimum value of $\bm{L}$ is denoted as $\bm{L}^*$.
  
As illustrated in Fig. \ref{ltol}, $\bm{L}$ is considered as a series of elliptic
level curves with $\bar{\bm{P}}$ as the center of each ellipse.
If $\bar{\bm{P}}$ lies within the domain of (${\xi}, {\tau}$) (colored in green),
we can obtain $\bm{L}^* = 0$, corresponding to $\bm{D}^* = \bm{D}_0$. Otherwise,
the minimum of $\bm{L}$ must be attained at the boundary of the domain.
Thus, the problem of computing $\bm{L}^*$ is converted to determining the critical point $\bm{P}$
(depicted as red points in the Fig. \ref{ltol}) where the elliptic function $\bm{L}(\xi, \tau)$ first
intersects the domain of (${\xi}, {\tau}$). The detailed process of calculating $\bm{P}$ 
can be determined by simultaneously solving the domain constraints and the elliptic function equation.
The process of computing $\bm{D}^*$ is outlined in \textbf{Algorithm \ref{region012solution}}.
This algorithm demonstrates the conversion of segment-to-segment distance into
either point-to-segment distance or point-to-point distance through an analysis of the
geometric relationship between elliptic contours and the domain.
\begin{algorithm}
  \caption{Minimum segment-to-segment distance in 3D space.}
  \label{region012solution}
  \KwData{Endpoints $(\bm{A}_0, \bm{A}_1)$ and $(\bm{B}_0, \bm{B}_1)$}
  \KwResult{Minimum value $\bm{D}^*$ }
  Calculate $z_1,z_2,z_3,z_4,z_5,z_6$ by (\ref{Ds2swhere});\\
  Calculate $\Delta, \bm{D}_0, \bm{L}$;\\
  \If{ $\Delta == 0$}
  {
    $\bm{D}^*  \leftarrow    \min \{\bm{D}(\bm{A}_0,\bm{l}_B) , \bm{D}(\bm{A}_1,\bm{l}_B) \};$
  }
  \Else
  {
    $\bar{\bm{P}}  \leftarrow$ caculate center of ellipse ($\bar{\xi}, \bar{\tau}$);\\
    \If{$\bar{\bm{P}}$ lies in domain of (${\xi}, {\tau}$)}
     {$ \bm{D}^*  \leftarrow \bm{D}_0$;}
    \Else 
      {
        Calculate the critical point $\bm{P}$ and corresponding coordinate (${\xi}^*, {\tau}^*$);\\
        $\bm{D}^*  \leftarrow$ $\bm{D}({\xi}^*, {\tau}^*)$;\\
      }
  }  
  \end{algorithm}
  
\subsubsection {Segment-to-rectangle Distance in 3D Space}
Computing the minimum segment-to-rectangle distance in 3D space is considerably more intricate than that between line segments.
This depends on all the distance calculation methods mentioned before.

Consider a line segment $\bm{l}$ defined by endpoints ($\bm{P}_0$,$\bm{P}_1$) 
and a 3D rectangle $\bm{\mathcal{R}} \in \mathbb{R}^{12}$ 
defined by four vertices $\bm{V}_i\in\mathbb{R}^3,i = 0,1,2,3$.
\begin{figure}
  \centering
  \includegraphics[width=0.7\linewidth]{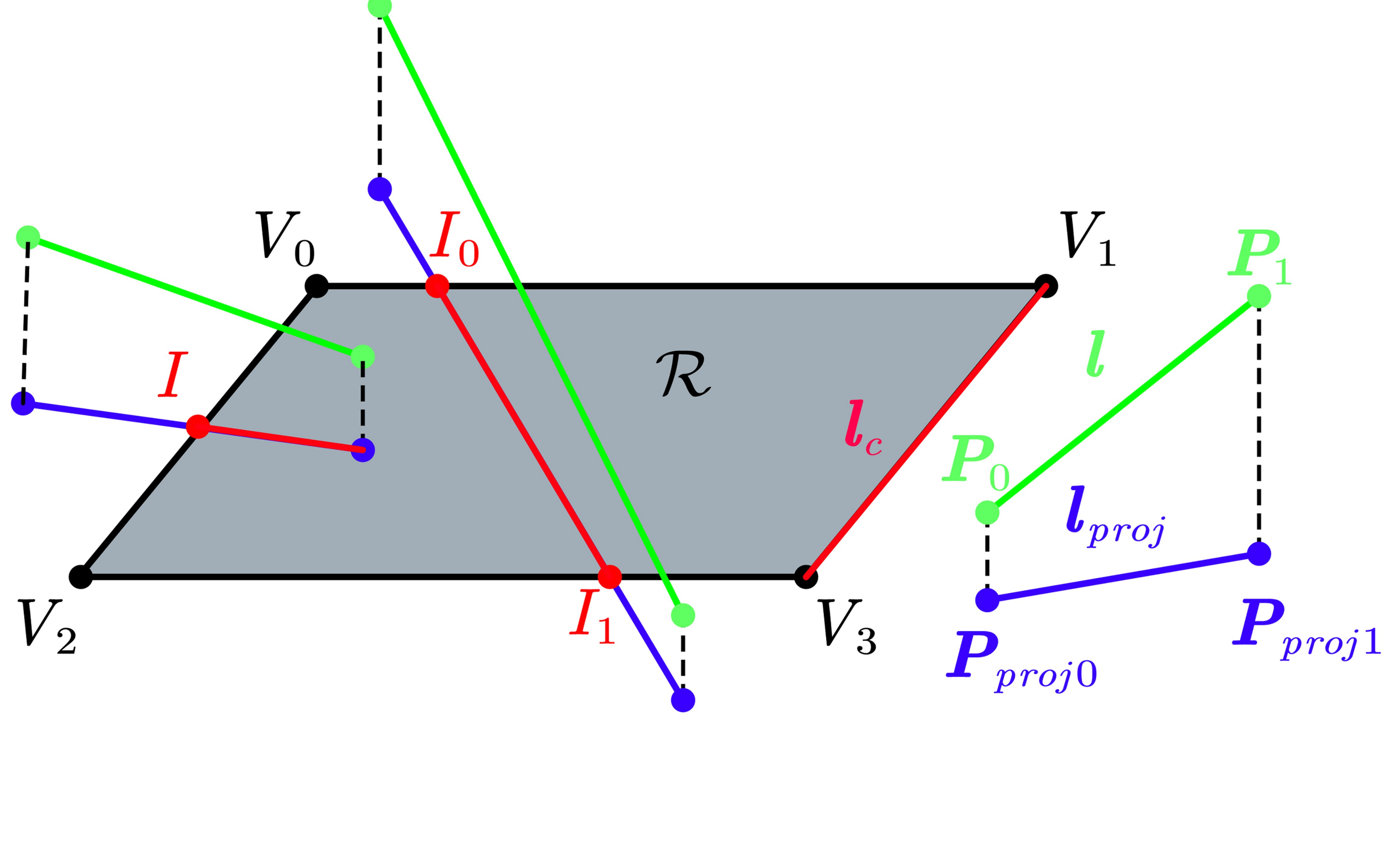}
  \caption{Illustration of the three types of projections of green line segments $\bm{l}$
  onto the plane containing rectangle $\bm{\mathcal{R}}$. 
  And the red line segments $\bm{l}_c$ are the closest line segments within $\bm{\mathcal{R}}$ to $\bm{l}$.}
  \label{ltor}
\end{figure}
Since the algorithm of minimum segment-to-segment distance has already been derived, 
we can identify a line segment $\bm{l}_c$ (subset of the rectangle) within $\bm{\mathcal{R}}$, 
such that the shortest distance from $\bm{l}$ to $\bm{l}_c$ is 
also the shortest distance from $\bm{l}$ to $\bm{\mathcal{R}}$.
We determine $\bm{l}_c$ through the projection method.
As shown in the Fig. \ref{ltor}, the projection of $\bm{l}$ onto the plane containing $\bm{\mathcal{R}}$ 
results in a line segment $\bm{l}_{proj}$ with endpoints ($\bm{P}_{proj0} , \bm{P}_{proj1}$).
There are three potential positional relationships between $\bm{l}_{{proj}}$ and $\bm{\mathcal{R}}$:
\begin{itemize}
  \item $\bm{l}_{{proj}}$ may not intersect with $\bm{\mathcal{R}}$.
  \item $\bm{l}_{{proj}}$ may intersect with $\bm{\mathcal{R}}$, 
  resulting in a single intersection point $\bm{I}$.
  \item $\bm{l}_{{proj}}$ may intersect with $\bm{\mathcal{R}}$, 
  resulting in two intersection points $\bm{I}_0$ and $\bm{I}_1$.
\end{itemize}
The number of intersection points is denoted as $n$.
A succinct method for calculating the minimum segment-to-rectangle
distance is outlined in \textbf{Algorithm \ref{reclinesolution}}.
\begin{algorithm}
  \caption{Minimum segment-to-rectangle distance in 3D space.}
  \label{reclinesolution}
  \KwData{Line segment $\bm{l}$, 3D rectangle $\bm{\mathcal{R}}$}
  \KwResult{Minimum value $\bm{D}^*$ }
  $\bm{l}_{proj}$ $\leftarrow$ project $\bm{l}$ onto the plane containing $\bm{\mathcal{R}}$ \;
  $n,\bm{I},\bm{I}_0,\bm{I}_1$ $\leftarrow$ caculate intersection points of $\bm{l}_{proj}$ and $\bm{\mathcal{R}}$\;
  
  \If{n == 0}
  {
  $\bm{l}_c$  $\leftarrow$ the closest edge of $\bm{\mathcal{R}}$ to $\bm{l}$\;
  }
  \If{n == 1}
  {
  $\bm{l}_c$  $\leftarrow$ line segment ($\bm{I},\bm{P}_{proj0}$) or ($\bm{I},\bm{P}_{proj1}$)\;
  }
  \If{n == 2}
  {
  $\bm{l}_c$  $\leftarrow$ line segment ($\bm{I}_0,\bm{I}_1$)\;
  }
  $\bm{D}^*$  $\leftarrow$ caculate minimum distance of $(\bm{l}_c,\bm{l})$ 
  by \textbf{Algorithm \ref{region012solution}}
  \end{algorithm}
\subsubsection {Bounding Box of Cuboid}
As a cuboid is composed of six 3D rectangles, the computation of six segment-to-rectangle distances
per cuboid can significantly impact computing efficiency. 
To address this problem, the paper presents a novel bounding box for cuboid.

For any cuboid, where the length is $a_r$, the width is $b_r$, 
and the height is $h_r$, we can establish the condition $a_r \geq b_r \geq h_r$.
Then, we define a rectangle $\bm{\mathcal{R}}_2$ in such a way that its four edges
are parallel to the length or width edges of the cuboid. Additionally,
its geometric center coincides with the center of the cuboid,
as illustrated by the yellow rectangle $\bm{\mathcal{R}}_2$ in Fig. \ref{Bounding}.
A view perpendicular to $\bm{\mathcal{R}}_2$ is depicted in Fig. \ref{verticalview}.
The distance from $\bm{\mathcal{R}}_2$ to $\bm{\mathcal{R}}_1$ is denoted as $d_{re}$, where $d_{re} \in [0, \frac{b_r}{2}]$.
Here, the length and width of $\bm{\mathcal{R}}_2$ are $x_r = a_r - 2d_{re}$ and $y_r = b_r - 2d_{re}$, respectively.
The bounding box can be obtained by extending the rectangle $\bm{\mathcal{R}}_2$ by a distance of 
$r_e = \sqrt{2d_{re}^2+h_r^2/4}$,
which represents the length of the closest pair of corners between $\bm{\mathcal{R}}_2$ and the cuboid's vertices.
The volume of the bounding box is calculated by:
\begin{equation}
\label{BoundingV}
\begin{aligned}
V_b&=2x_ry_rr_e+\pi r_e^2(x_r+y_r)+\frac{4}{3}\pi r_e^3.\\
\end{aligned}
\end{equation}
As $\bm{\mathcal{R}}_2$ varies with $d_{re}$, it can span from the yellow dashed line in Fig. \ref{verticalview}
at its minimum to overlapping with $\bm{\mathcal{R}}_1$ at its maximum.
The volume of the resulting bounding box varies accordingly.
The optimal bounding box can be determined by finding the root of $\frac{\partial V_b}{\partial d_{re}} = 0$
through numerical solver. This bounding box slightly occupies free space but significantly enhances computing speed.
If the robot needs to avoid the cuboid accurately, it can also choose to consider six rectangles.
\begin{figure}
\centering
\includegraphics[width=0.68\linewidth,trim=15cm 3.8cm 12cm 4.0cm, clip]{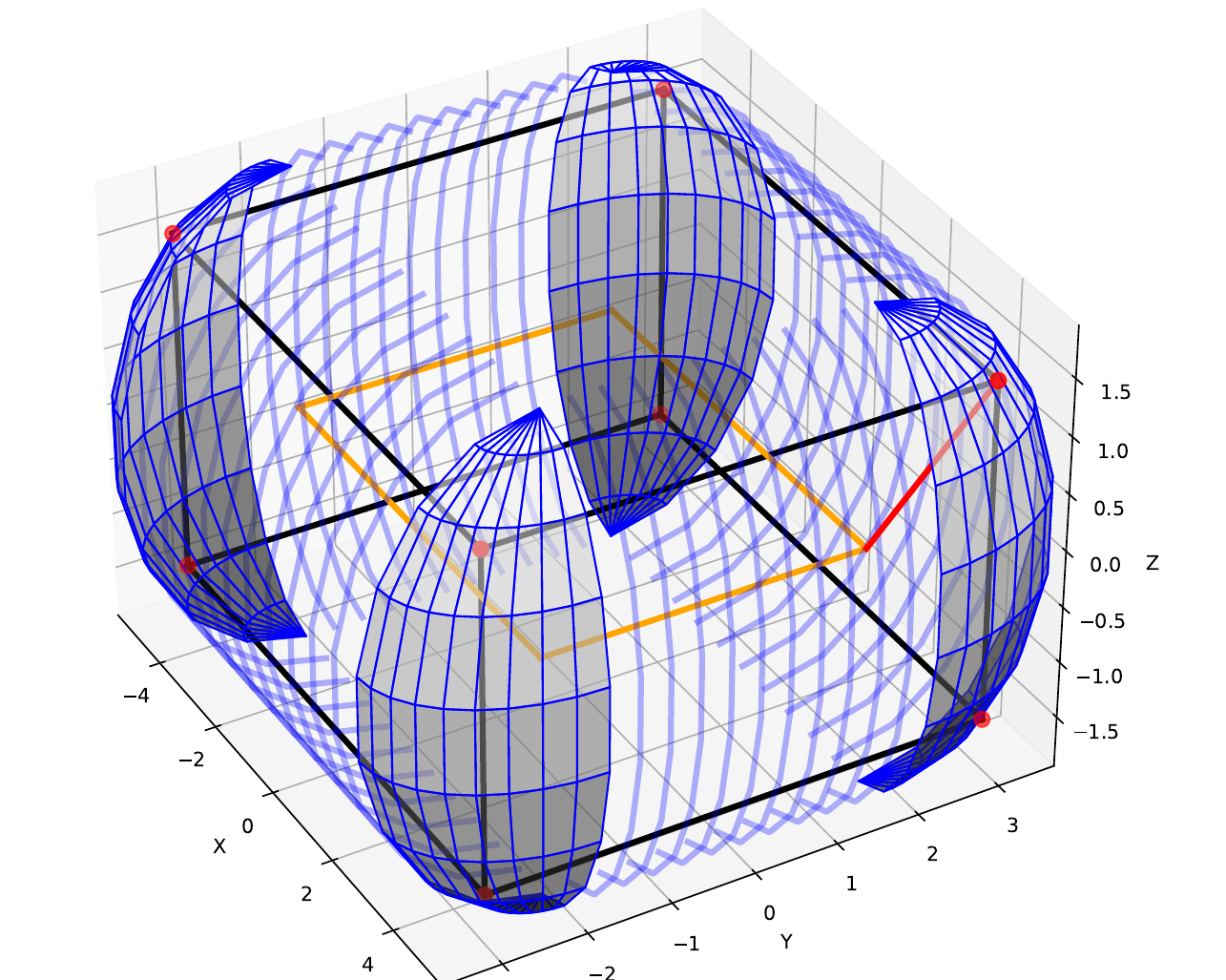}
\caption{ Bounding Box for Cuboid.}
\label{Bounding}
\end{figure}
\begin{figure}
\centering
\includegraphics[width=0.58\linewidth]{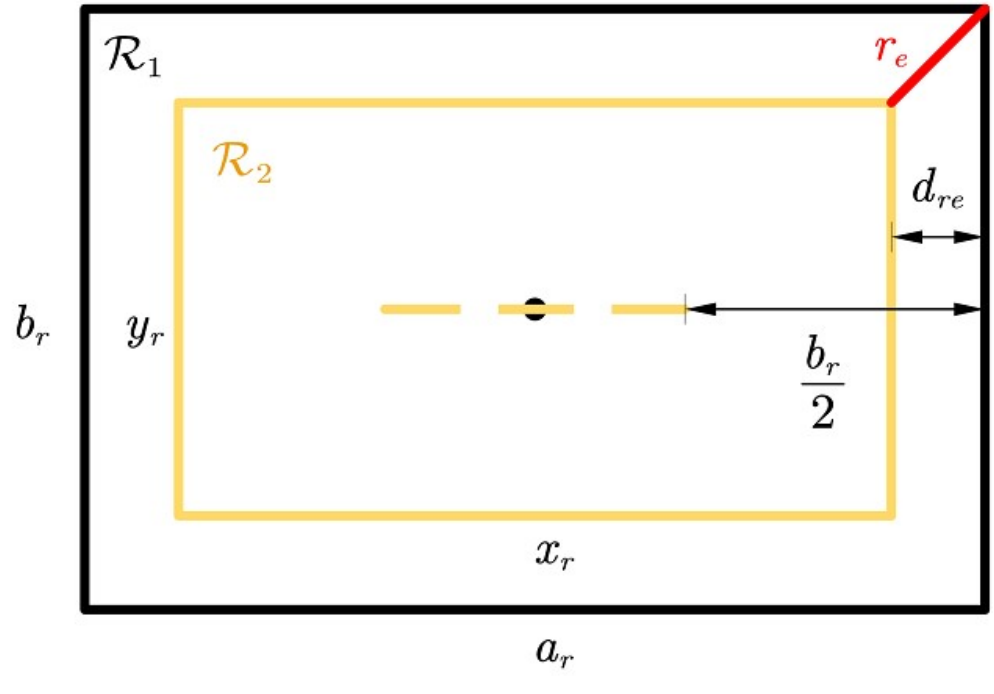}
\caption{ Vertical view of cuboid $\bm{\mathcal{R}}_1$ and rectangle $\bm{\mathcal{R}}_2$.}
\label{verticalview}
\end{figure}

\subsection{RDCBF Design for Mobile Manipulator}

Let define the system state by the positions and joint angles,
while the output of the RDCBF planner comprises velocity commands. 
We consider a simplified version of \eqref{sysdeq} with matched disturbance $\bm{d}$.
Therefore, the system can be succinctly expressed as:
\begin{equation}
\label{systemforplanning}
\dot{\bm{x}} = \bm{u} + \bm{d}.
\end{equation}

To ensure the safety of the entire robot, it is divided into multiple components labeled as 
$\mathcal{L}_i, i\in \mathbb{N}^+ $. The function $\mathcal{L}=[\mathcal{L}_1,\mathcal{L}_2,\cdots]^T$ relies on 
the system state and requires forward kinematics of the robot. Obstacle positions and velocities are represented 
by $\bm{p}_j$ and $\bm{v}_j$, respectively, where $j$ also belongs to positive integers. In Sec.~\ref{sec:disturance}, 
the methods for calculating distances between different obstacles have been discussed; 
however, distance alone cannot determine system safety.

\begin{assumption}
Assuming that the velocities of obstacles and robot remain constant within a planning period 
$T$ (which is a reasonable assumption for short $T$).
\end{assumption}

The safety of the system (\ref{systemforplanning}) is influenced by both
system state $(\bm{x},\dot{\bm{x}})$ 
and obstacles state $(\bm{p},\bm{v})$.
We define the safety function $h(\cdot)$ for trajectory planning in a dynamic environment as follows:
\begin{equation}
    \label{RDCBFforplaning}
    h_{ij}=h(\mathcal{L}_i(\bm{x},\dot{\bm{x}}),\bm{p}_j,\bm{v}_j)=\bm{D}_{ij}
    (\mathcal{L}_i,\bm{p}_j+T\bm{v}_j)-R_{ij}^2, 
\end{equation}
where $i$ and $j$ represent the indices of the components of robot and obstacles, respectively.
$\mathcal{L}_i(\bm{x},\dot{\bm{x}})=\mathcal{L}_i(\bm{x}+T\dot{\bm{x}})$ represents the state of the $i$-th component after time $T$, 
$\bm{D}_{ij}$, utilizing \textbf{Algorithm \ref{region012solution}} and \textbf{Algorithm \ref{reclinesolution}},
represents minimum square distance between $\mathcal{L}_i$ and $j$-th obstacle, 
and $R_{ij}$ the safe distance between them. As show in (\ref{RDCBFforplaning}), 
$h(\cdot)$ performs a one-step forward prediction 
of the state over time $T$ to avoid collisions during online planning. 
Let $\hat{\bm{v}}$ be the measured velocity. We define the velocity error of the obstacle as below:
\begin{equation}
\label{verror}
\bm{e}_v = \hat{\bm{v}} - \bm{v},\|\bm{e}_v\|_2\leq\epsilon_v.
\end{equation}
To ensuring safety under disturbance and measurement error of obstacle's velocity,
we propose the RDCBF theorem as follows:
\begin{theorem}
Consider the dynamic safe function $h(\cdot)$ in (\ref{RDCBFforplaning}), the system \eqref{sysdeq}
and the DOB given in~(\ref{DOB}) with $\hat{\bm{d}}(0)=\bm{0}$. Suppose that Assumption 1 and 2 hold, 
and $h_0 > 0$. If there exists constants $\gamma,\alpha,\beta >0$ such that $\alpha > \frac{\gamma+\mu}{2}$, 
$\beta > \frac{\|\bm{e}_d(0)\|^2}{2h_0}$, and 
\begin{equation*}
\sup_{\bm{u} \in \mathcal{U}}
\left[
L_fh 
- \|L_{g_{d}}h\|\Gamma 
- \frac{\omega_1^2}{2\mu\beta} 
- \chi  
+ \Lambda
+ L_gh\bm{u}
+ \gamma h
\right] 
\geq 0
\end{equation*}
 hold true, where 
$L_{g_{d}}h = \frac{\partial h}{\partial \bm{x}}\bm{g}_{d}$,
$\Gamma = \omega_0 + \sqrt{\omega_0^2+\frac{\omega_1^2}{2\mu \kappa}}$, 
$\chi = \frac{\beta\|L_{g_{d}}h\|^2}{4\alpha-2\mu-2\gamma}$ and 
$\Lambda=\frac{\partial h}{\partial \bm{p}}\hat{\bm{v}} 
- \left\|\frac{\partial h}{\partial \bm{p}}\right\|_1\epsilon_v$.
Then any Lipschitz continuous controller 
$\bm{u}(\bm{x},\bm{p},\hat{\bm{d}},\hat{\bm{v}}) \in
K_{rdcbf}(\bm{x},\bm{p},\hat{\bm{d}},\hat{\bm{v}})=\{ \bm{u} \in \mathcal{U}: 
\phi + L_gh\bm{u}\geq 0 \}
$, where
$\phi = L_fh+L_{g_{d}}h\hat{\bm{d}} - \frac{\omega_1^2}{2\mu\beta}-\chi+\gamma h + \Lambda$
will guarantee $h(t) \geq 0$ for all $t \geq 0$.
\end{theorem}

\begin{proof}
If $\hat{\bm{d}}(0)=0$, then 
$\|\hat{\bm{d}}(t)\|=\|{\bm{d}}(t)+\bm{e}_d(t)\|<\Gamma$ 
from (\ref{etconverge}). Define a candidate RDCBF $\bar{h}$ as $\bar{h}=\beta h - \frac{1}{2} \bm{e}_d^T\bm{e}_d$.
One can see that $\bar{h}\geq0$ implies $h\geq\frac{\|\bm{e}_d\|^2}{2\beta}\geq0$. 
At $t=0$, $\bar{h}(0)>0$ since $\beta > \frac{\|\bm{e}_d(0)\|^2}{2h_0}$. 
Moreover,
\begin{equation*}
    \begin{aligned}
        \dot{\bar{h}}&=\beta\dot{h}-\dot{V}\\
    &\geq \beta(L_fh+L_gh\bm{u}+L_{g_d}h\bm{d}+\frac{\partial h}{\partial \bm{p}}\bm{v})
    +\kappa\bm{e}_d^T\bm{e}_d-\frac{\omega_1^2}{2\mu}\\
    &=\beta(L_fh+L_gh\bm{u}+L_{g_d}h\hat{\bm{d}}+\frac{\partial h}{\partial \bm{p}}\bm{v}) - \frac{\omega_1^2}{2\mu}
    +\frac{\gamma}{2}\bm{e}_d^T\bm{e}_d\\
    &\quad +(\alpha-\gamma/2-\mu/2)\bm{e}_d^T\bm{e}_d -\beta L_{g_d}h\bm{e}_d+\beta\chi-\beta\chi\\
    &=\beta(L_fh+L_gh\bm{u}+L_{g_d}h\hat{\bm{d}}+\frac{\partial h}{\partial \bm{p}}\bm{v}-\chi) - \frac{\omega_1^2}{2\mu}
    +\frac{\gamma}{2}\bm{e}_d^T\bm{e}_d\\
    &\quad + 
    \left\|\sqrt{(\alpha-\gamma/2-\mu/2)}\bm{e}_d^T-\frac{\beta L_{g_d}h}{2\sqrt{(\alpha-\gamma/2-\mu/2)}}\right\|_2^2\\
    &\geq \beta(L_fh+L_gh\bm{u}+L_{g_d}h\hat{\bm{d}}+\frac{\partial h}{\partial \bm{p}}\bm{v}-\chi) - \frac{\omega_1^2}{2\mu}
    +\frac{\gamma}{2}\bm{e}_d^T\bm{e}_d\\
    &\geq \beta(L_fh+L_gh\bm{u}+L_{g_d}h\hat{\bm{d}}+\frac{\partial h}{\partial \bm{p}}\hat{\bm{v}}
    -\left\|\frac{\partial h}{\partial \bm{p}}\right\|_1\epsilon_v-\chi)\\
    &\quad - \frac{\omega_1^2}{2\mu}
    +\frac{\gamma}{2}\bm{e}_d^T\bm{e}_d=\beta(\phi + L_gh\bm{u}-\gamma h)+\frac{\gamma}{2}\bm{e}_d^T\bm{e}_d\\
    &\geq -\gamma(\beta h - \frac{1}{2}\bm{e}_d^T\bm{e}_d)
    \end{aligned}
\end{equation*}
Therefore, any $\bm{u}\in K_{rdcbf}$ yields $\dot{\bar{h}}\geq-\gamma\bar{h}$, which implies
that $\bar{h}(t) \geq 0, \forall t\geq0$ as $\bar{h}(0) \geq 0$. Thus, $h(t)\geq0,\forall t\geq0$ 
\end{proof}

The structure of RDCBF-based planner is depicted in Fig.\ref{structure}. 
The global planner provides velocity commands $\bm{u}_{nom}$,
and the RDCBF planner modifies these commands to ensure system safety.

\section{Case Study}
\subsection{Problem Statement}
This work aims to incorporate RDCBF in trajectory planning tasks to ensure safety in complex dynamic environments,
even in the presence of external disturbance and the measured error of obstacle velocity.
For broader applicability, we model the problem kinematically to provide velocity commands,
similar to other methods such as DWA \cite{DWA}, Kinetostatic Danger Field \cite{KDF}, or APF \cite{APFarm}.
As the RDCBF planner ensures system safety based on the kinematic model,
it can seamlessly integrate into almost any global planner, including GPMP, RRT, Reinforcement Learning, 
or even a simple trajectory tracking controller.
There are two mild conditions for our methods:
\begin{itemize}
\item The positions and safe distances of obstacles are assumed to be known or measurable,
and the measurement errors of obstacle velocities are bounded.
\item A nominal controller (global planner) exists to reach the goal.
\end{itemize}

\indent 
In this paper, we applied the RDCBF trajectory planner to an 8-DoF mobile manipulator.
We designed two scenarios to test the obstacle avoidance capabilities,
trajectory qualities, and the system safety under disturbance and inaccurate obstacle information.
Additionally, there are constraints regarding the environment boundaries and the avoidance of self-collision.
These constraints are not detailed further in this section.

Commonly used methods in local planning of mobile manipulator, 
such as APF \cite{APF}, MPC \cite{MPC3}, and CBF-based methods proposed in this paper, 
will be evaluated in both scenarios. 
Notably, \textbf{Algorithm \ref{region012solution}} and \textbf{Algorithm \ref{reclinesolution}} 
are also employed in MPC-based methods, serving as the distance algorithms for obstacle constraints.

\begin{itemize}
\item Scenario A: Shown in Fig.~\ref{twoscenarios:sub1},
this scenario features two spherical and two cylindrical obstacles with adjustable velocities.
More than 30 safety constraints need to be considered. 
The red circle on the floor represents the goal.
The mobile manipulator starts from coordinates (0,0) and moves towards the goal at coordinates (0,2).
Simultaneously, obstacles approach the mobile manipulator.

\item Scenario B: Shown in Fig.~\ref{twoscenarios:sub2},
Involving 22 obstacles, including flat and elongated cuboidal types,
the mobile manipulator navigates to green circle goal A (-1,0)
to grasp an object. It then maneuvers through a cluttered,
dynamic environment to reach red circle goal B (0,6).
More than 150 safety constraints need to be considered.
External disturbances are introduced to the manipulator,
and inaccurate obstacle information is intentionally provided to the planner.
Additionally, an ablation study will be conducted to assess the effectiveness of RDCBF.
\end{itemize}
    
\begin{figure}
    \centering
    \subfigure[Scenario A]{
    \includegraphics[width=0.225\textwidth]{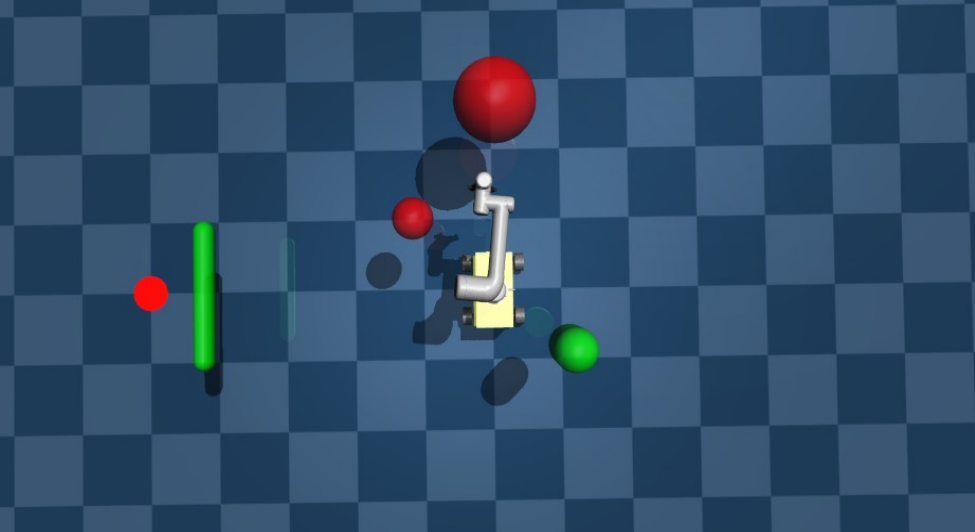}
    \label{twoscenarios:sub1}
    }
    \subfigure[Scenario B]{
    \includegraphics[width=0.225\textwidth]{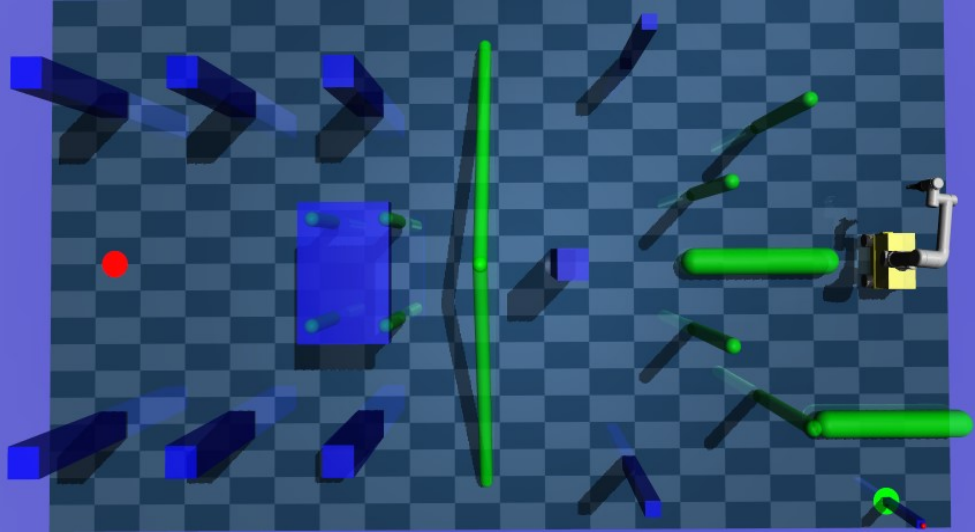}
    \label{twoscenarios:sub2}
    }
    \caption{Two experiment scenarios in Mujoco}
    \label{twoscenarios}
    \end{figure}
    
\indent All methods are implemented in Python, utilizing CasADi \cite{CasADi} 
as the optimization tool and Mujoco \cite{MuJoCo} as the physics engine.
\subsection{Experiment Results}
\subsubsection{Scenario A}
The experiment results are presented in Table \ref{tabSA}. 
The table provides the maximum obstacle velocity 
that the methods can handle without collisions (Max obs. V).
To compare the trajectories quality, it also presents key metrics including 
the planning frequency (Frequency), total time (Time), 
and trajectory length (Length) under the same condition 
where the maximum velocity of obstacles is set to 0.4m/s.
\begin{figure}
  \centering
  \includegraphics[width=0.78\hsize,trim=1.3cm 0.42cm 2cm 1.5cm, clip]{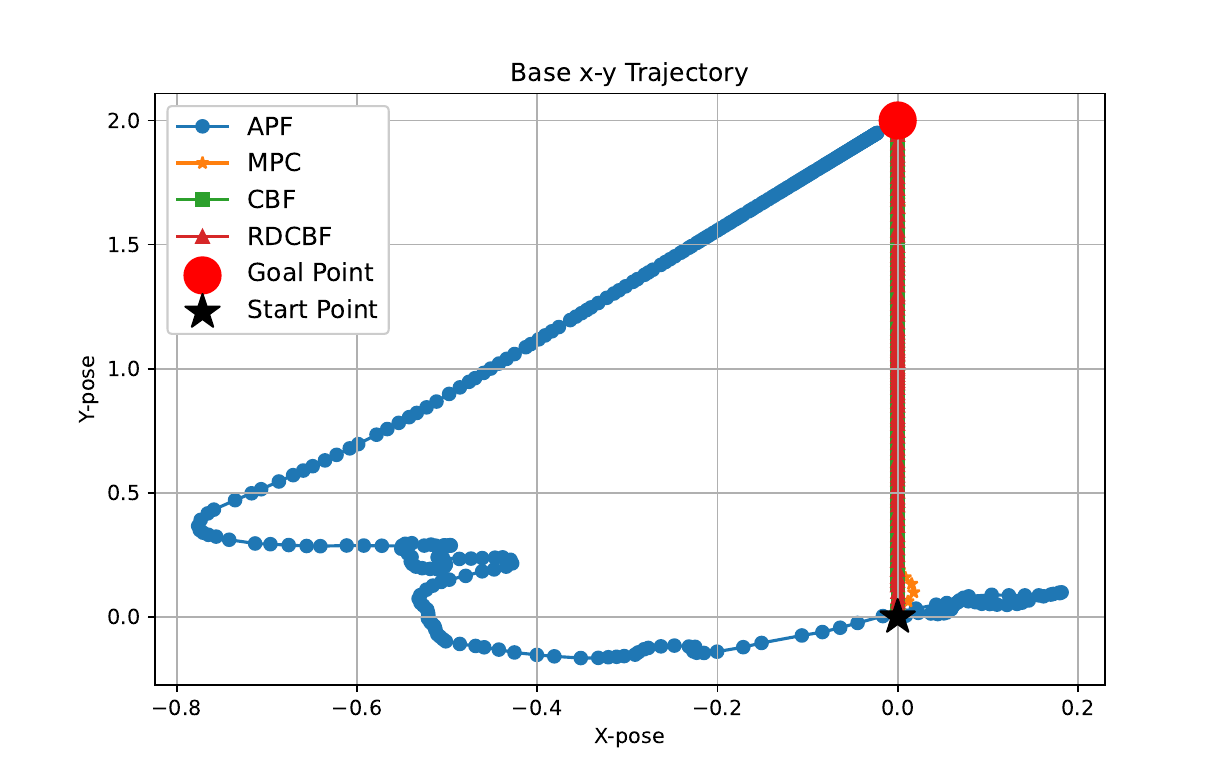}
  \caption{The paths of different methods.}
  \label{figSB2}
  \end{figure}
  \begin{figure}
    \centering
    
    \subfigure[APF method]{
    \includegraphics[width=0.22\textwidth]{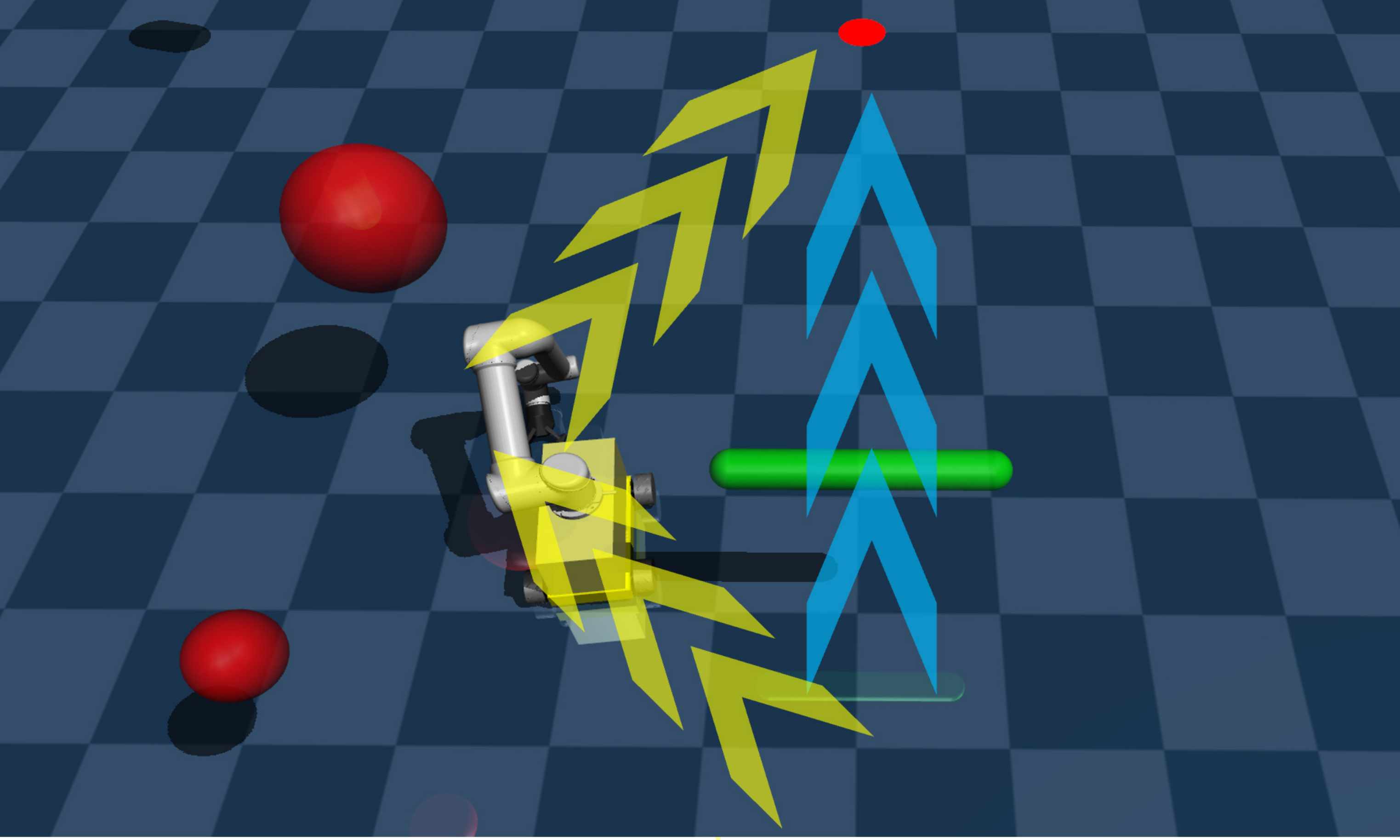}
    \label{SA1}
    }
    \subfigure[CBF- and MPC- based methods.]{
    \includegraphics[width=0.22\textwidth]{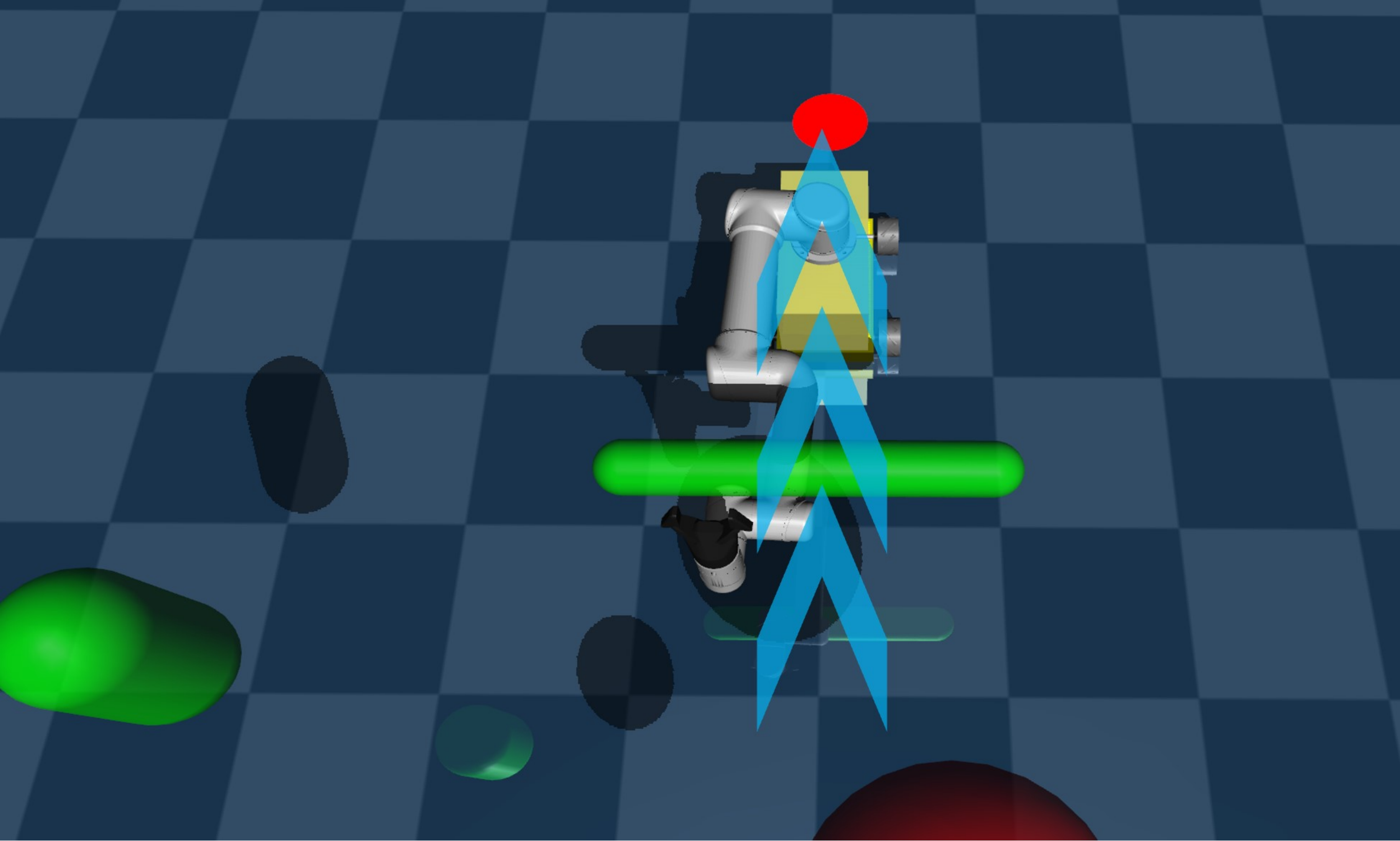}
    \label{SA2}
    }
    \caption{Planning Srategy of different methods. 
    The blue arrow represents the optimal trajectory, while
    the yellow arrow represents a detour.}
    \label{SA}
    \end{figure}

\begin{table}
    \centering
    \caption{Simulation Results of Scenario A}
    \label{tabSA}
    \begin{tabular}{ccccc}
        \toprule
         Methods       & Max obs. V& Frequency  & Time& Length\\
        \midrule
        APF        & 0.560 m/s & \textbf{335.57 Hz} & 7.5794 s&  3.895 m\\
        MPC(N=2)   & 0.500 m/s & 45.581 Hz          & 2.3574 s&  2.016 m\\
        MPC(N=3)   & 0.200 m/s & 16.611 Hz          & failed &  failed \\
        MPC(N=4)   & 0.046 m/s & 5.9789 Hz       & failed &  failed \\
        CBF        & 0.540 m/s & 261.92 Hz          &  2.2486 s&  2.009 m\\
        RDCBF       & \textbf{0.720 m/s} & 189.76 Hz & \textbf{2.2376 s}&  \textbf{2.008 m}\\
        \bottomrule
    \end{tabular}
\end{table}

\begin{figure*}
  \centering
  \subfigure[$t = 0$s: initial state.]{
    \includegraphics[width=.28\linewidth]{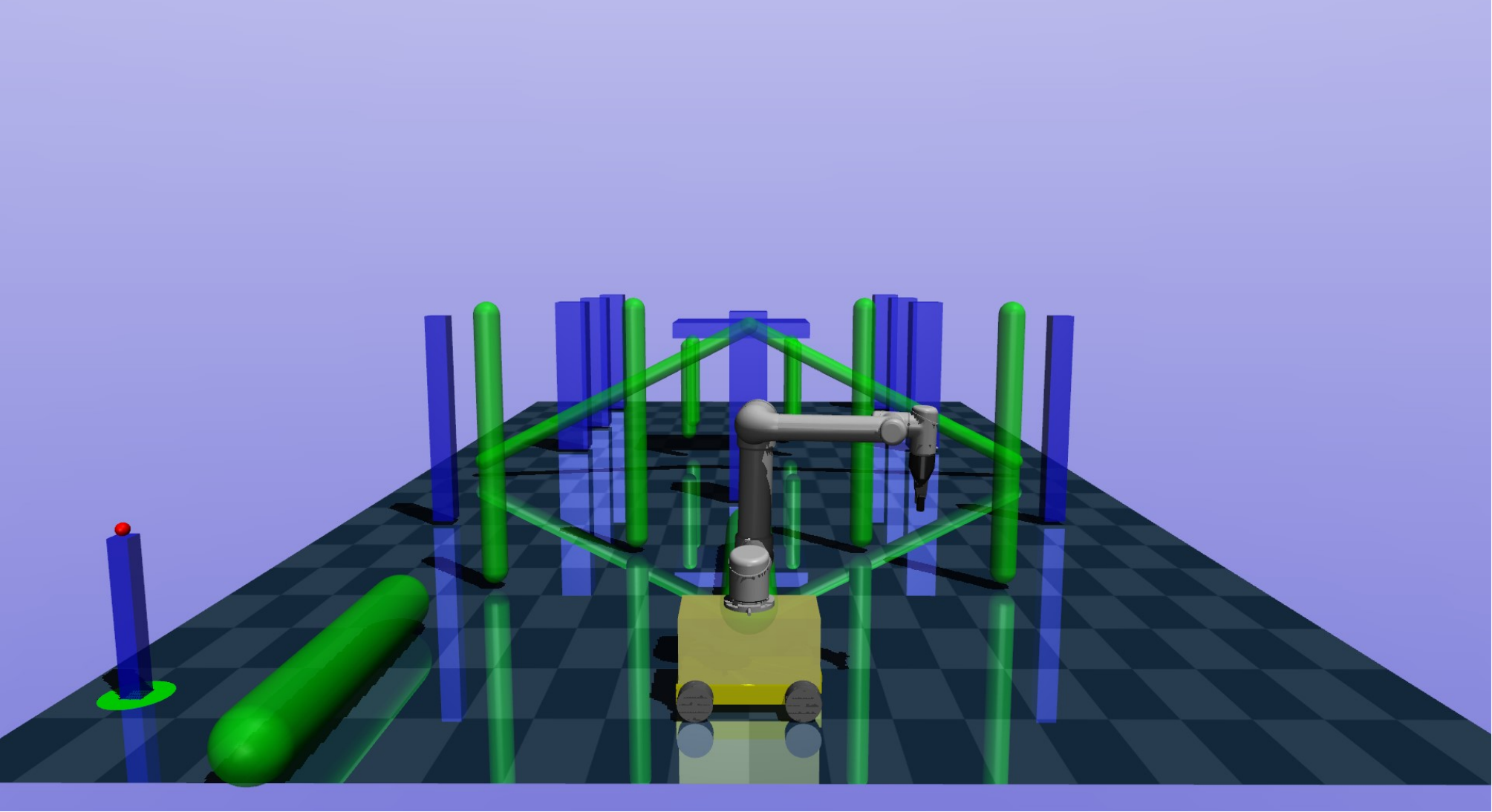}

    \label{fig:sub1}
  }
  \hfill
  \subfigure[$t = 8.5$s: grasps an object. ]{
    \includegraphics[width=.28\linewidth]{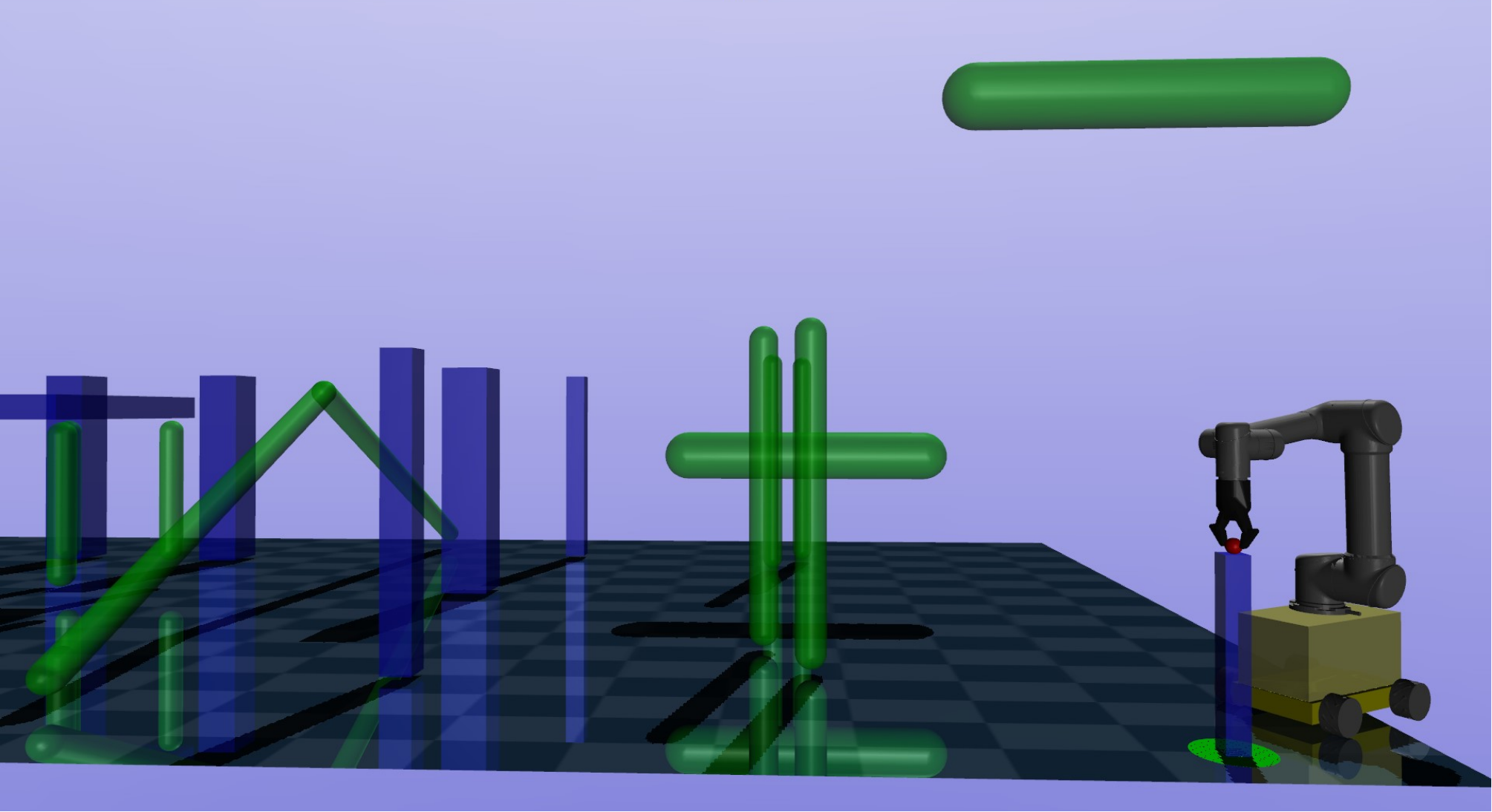}

    \label{fig:sub2}
  }
  \hfill
  \subfigure[$t = 11.5$s: navigates through numerous obstacles.]{
    \includegraphics[width=.28\linewidth]{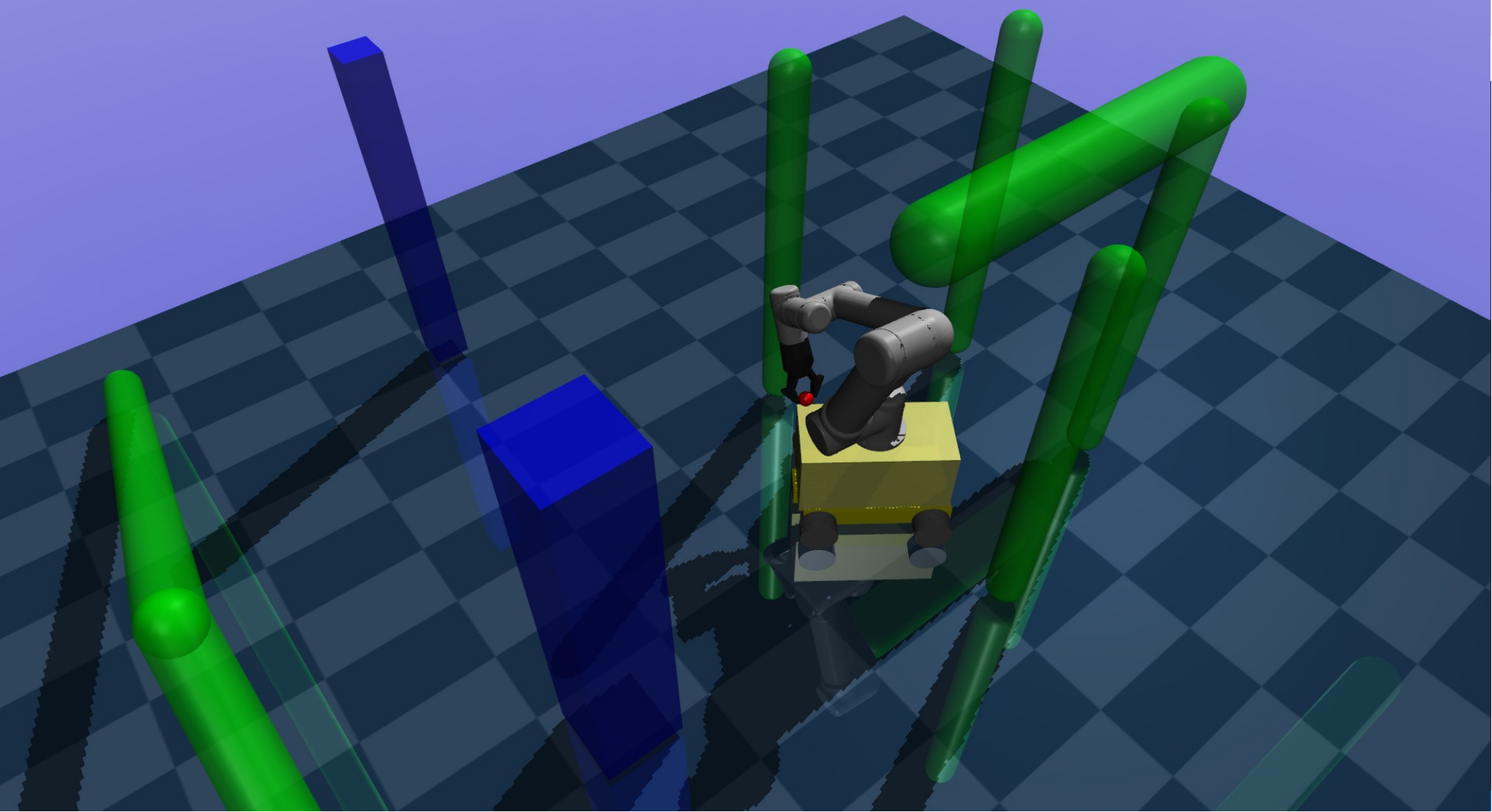}

    \label{fig:sub3}
  }
  
  \subfigure[$t = 14$s: navigates through a stationary archway in the shape of a triangle.]{
    \includegraphics[width=.28\linewidth]{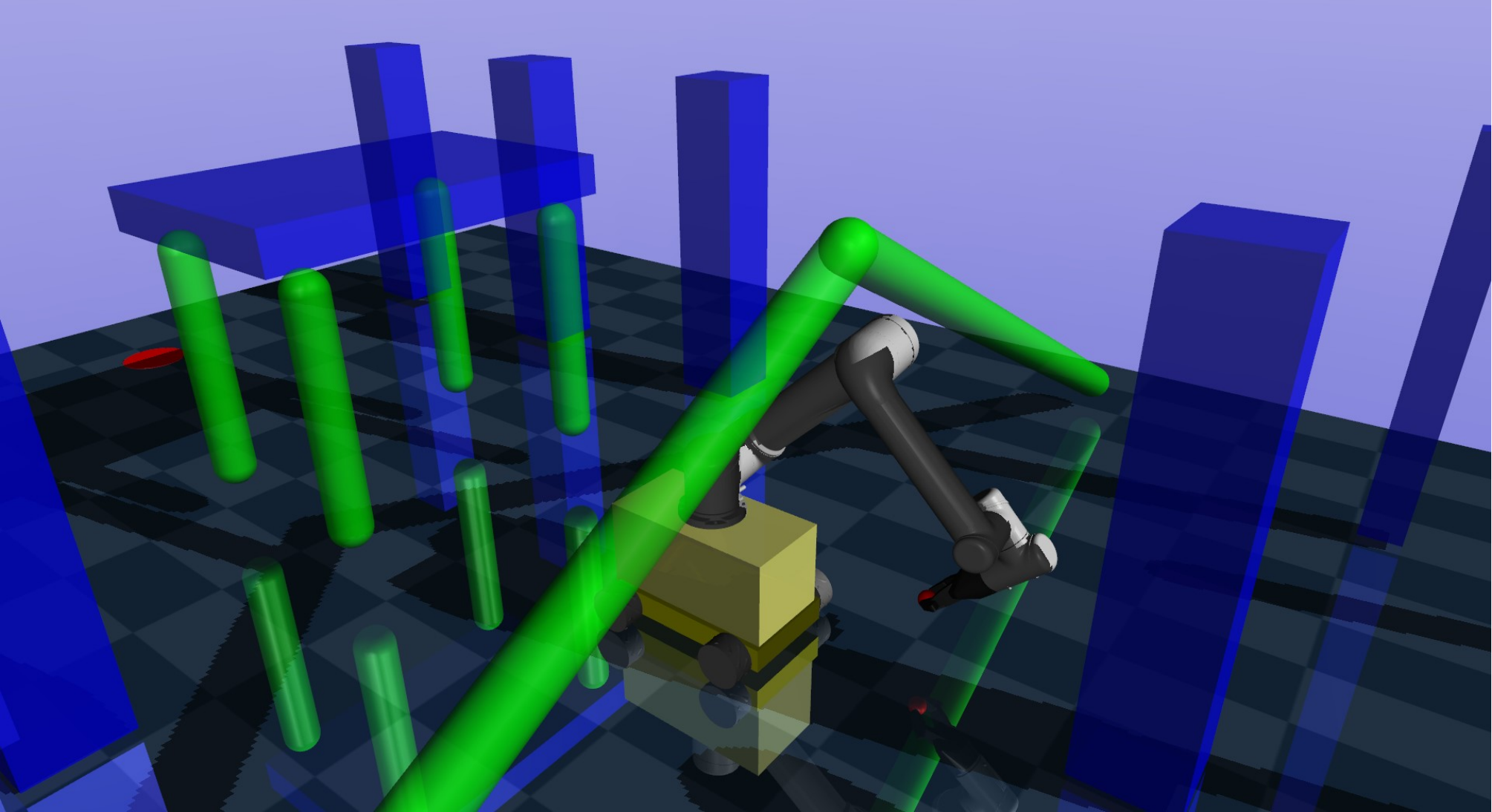}

    \label{fig:sub4}
  }
  \hfill
  \subfigure[$t = 16$s: navigates through a set of composite obstacles resembling a table.]{
    \includegraphics[width=.28\linewidth]{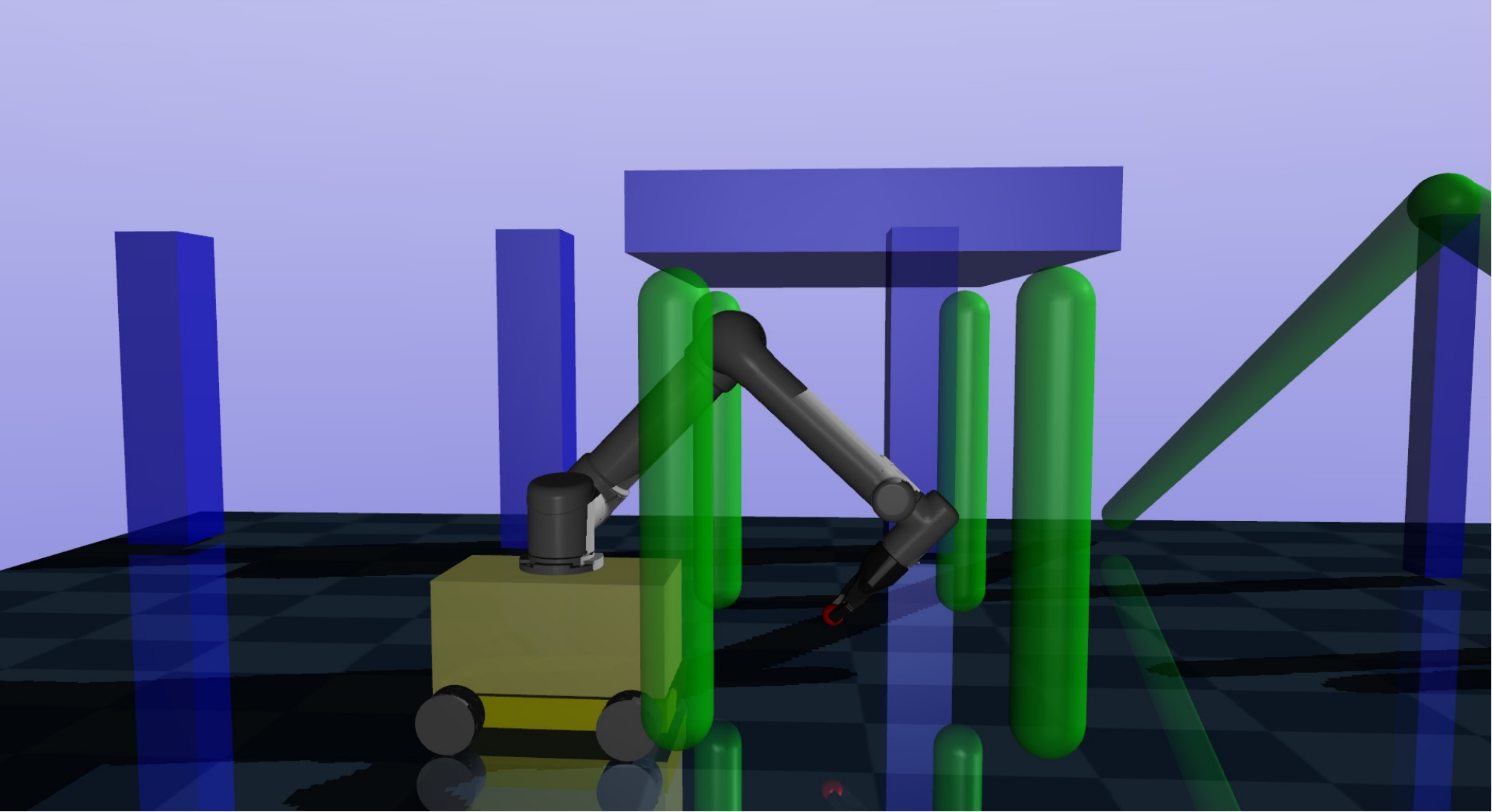}

    \label{fig:sub5}
  }
  \hfill
  \subfigure[$t = 18.5$s: final state.]{
    \includegraphics[width=.28\linewidth]{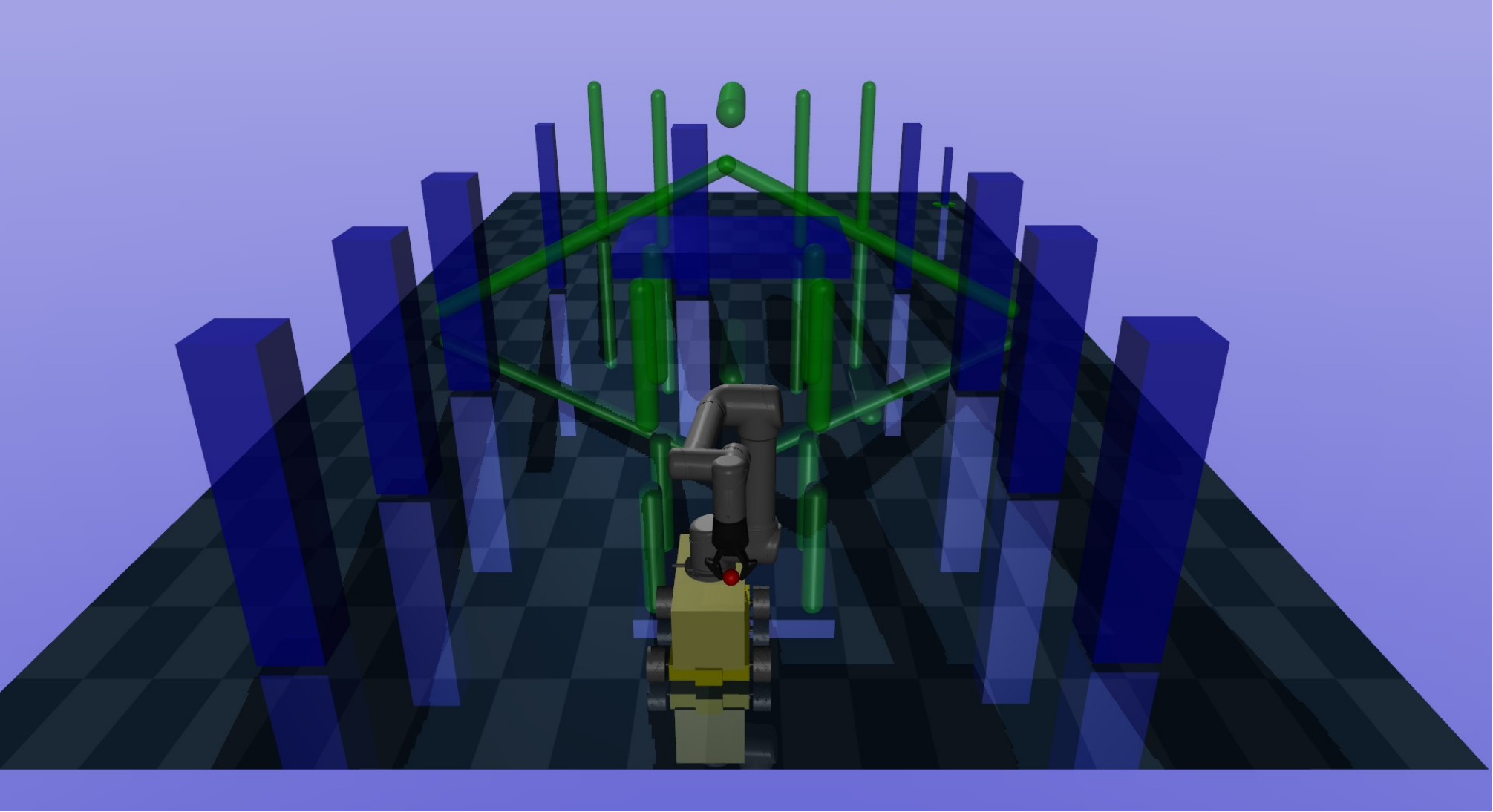}

    \label{fig:sub6}
  }
  
  \caption{Snapshots of mobile manipulator utilizing RDCBF planner.}
  \label{snapshots}
\end{figure*}
\indent The APF demonstrates the shortest computation time,
as it directly calculates the potential field function.
CBF-based methods require solving a QP problem and calculating more complex distance functions,
resulting in lower planning frequencies compared to APF.
However, since all constraints are linear constraints on $\bm{u}$,
the time cost of solving the QP problem is acceptable.
Additionally, MPC-based methods entail solving an N-step nonlinear optimization problem,
with time consumption increasing significantly as N grows.
To accommodate dynamic obstacles, MPC must treat obstacle information as an extension of the system state.
Consequently, MPC-based method are the slowest among all methods.
The maximum velocity that these methods can handle is primarily influenced by
both planning frequency and their ability to incorporate dynamic information into calculations.
APF and CBF struggle to incorporate dynamic information, while MPC (N=2) suffers from slow computation speeds,
resulting in suboptimal performance for both approaches.
MPC (N=3 and N=4) failed when the maximum obstacle velocity was set to 0.4 m/s,
attributed to the unacceptable planning frequency.
In contrast, RDCBF leverages the advantages of DCBF to effectively handle dynamic information
while maintaining an acceptable planning frequency.
As a result, RDCBF emerges as the most effective method capable of handling the fastest obstacles among all.

\begin{figure}
  \centering
  \includegraphics[width=0.9\linewidth,trim=3cm 1cm 3cm 2.8cm, clip]{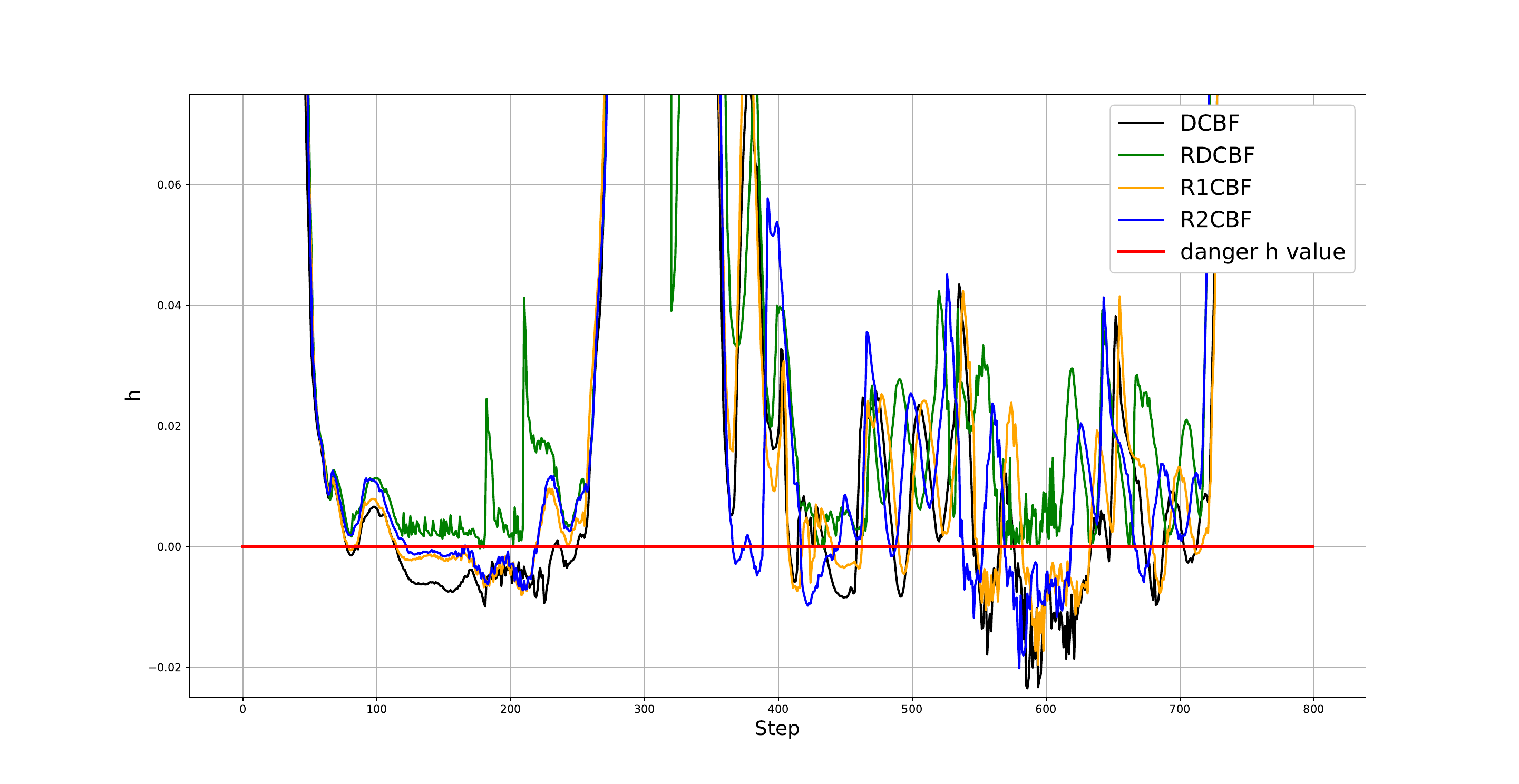}
  \caption{Safety function values of different methods.}
  \label{hvalue}
  \end{figure}

\indent The trajectory length and total time are additional crucial metrics considered by the planning methods.
The base paths of mobile manipulator are showned in Fig.~\ref{figSB2}. 
The APF method necessitates enclosing non-spherical obstacles with multiple spheres.
As illustrated in Fig.~\ref{SA1},
this leads to reduced free space and adopts an undesired strategy of stepping back and circumventing all obstacles.
The drawback of APF makes it more prone to falling into local optima, 
resulting in a longer trajectory length and total time.
In contrast, we observe that both CBF-based and MPC-based methods elegantly navigate through
the spaces between moving obstacles, as illustrated in Fig.~\ref{SA2}.
The paths of the manipulator base generated by these methods almost overlap,
forming straight lines connecting the start point and the goal.
Furthermore, the total time taken by these methods is also similar.
This similarity arises because all these methods utilize the same 3D space distance computing algorithm
introduced in our paper. This observation validates the effectiveness of the algorithm 
in reducing the occupation of free space.

\subsubsection{Scenario B}

In Scenario B, a trigonometric disturbance is applied to (\ref{systemforplanning}),
and the velocity of obstacles is inaccurate, 
with only 60\% of the real values. The maximum amplitude of the disturbance 
are $[0.15,0.20]$ m/s for base velocities, and $[0.5,0.5,0.15,0.35,0.25,0.45]$ rad/s for joint angles.

\indent 
All methods were tested 50 times in this scenario, 
but the APF method failed due to the enclosing of non-spherical obstacles with multiple spheres,
which blocked the trajectory to the target. Similarly, 
MPC also failed because of the excessive constraints and expanded state quantities, preventing real-time solving. 

\indent 
A series of ablation experiments were conducted based on RDCBF,
involving the elimination of different compensation terms:
only the compensation term for state disturbance (denoted as R1CBF),
only the compensation term for velocity error (denoted as R2CBF),
and finally, eliminating both compensation terms (denoted as DCBF).
Subsequently, these four methods were tested in Scenario B.
Experiment snapshots of the RDCBF algorithm are displayed in Fig. \ref{snapshots},
and Table \ref{tabSB} presents the planning frequency and success rate of all algorithms.

\begin{table}
    \centering
    \caption{Simulation Results of Scenario B}
    \label{tabSB}
    \begin{tabular}{ccc}
        \toprule
        Methods       &  Frequency  & Success Rate \\
        \midrule
        APF        & 66.85 Hz & failed\\
        MPC(N=2)   & 1.883 Hz & failed \\
        DCBF   & \textbf{69.39 Hz}& 26 \% \\
        R1CBF   & 56.40 Hz & 68 \% \\
        R2CBF        & 40.52 Hz & 76 \%\\
        RDCBF       & 35.12 Hz & \textbf{96} \%\\
        \bottomrule
    \end{tabular}
\end{table}

\indent 
The inclusion of two compensation terms approximately doubled the computation time,
yet the computational rate remains acceptable.
Activating only the dangerous constraints in the QP problem significantly accelerated the planning frequency.
In ideal conditions (without state disturbance and estimation uncertainty),
the planning frequency can even surpass that of APF.

\indent DCBF, lacking compensation for external disturbance and errors in obstacle information estimation,
exhibits the lowest success rate. Conversely, R1CBF and R2CBF, each eliminating one compensation term,
show relatively higher success rates compared to DCBF.
Ultimately, RDCBF achieves a success rate close to 100\%
by compensating for both disturbance and errors in obstacle information estimation.
The safety function values for CBF-based methods in 
this scenario are displayed in Fig. \ref{hvalue},
representing the minimum value among all safety constraints for the robot at the same time.
It is evident that RDCBF ensures safety in this complex dynamic environment under disturbance
and measurement uncertainty. In contrast, other ablation methods or traditional methods violate safety constraints.
This ablation experiment demonstrates the correctness and superior real-time adaptability
of the RDCBF algorithm in complex dynamic environments.
Since the dynamic trajectory planning task is closely related 
to the actutators limitations of the mobile manipulator and the speed of obstacles,
RDCBF may also fail in scenarios with ``no way out" or
facing excessively fast obstacles.

\section{Conclusion}
This paper proposes a novel path planning algorithm based on RDCBF,
enabling dynamic trajectory planning for high-dimensional mobile manipulators in complex dynamic environments.
The algorithm ensures real-time computation even under constraints involving up to 22 obstacles.
Notably, RDCBF can handle both system state disturbance and inaccurate environmental information,
a feature rarely seen in traditional path planning algorithms.
The algorithm presented in this paper focuses on velocity-level planning,
corresponding to the one relative order RDCBF algorithm.
Future work could extend to higher-order RDCBF algorithms to achieve a
robust acceleration planning framework in high-dimensional spaces,
considering both self-state and environmental information.
We also plan to apply this algorithm to real hardware in the future,
which will require integration with environmental perception technologies.

\end{document}